%% file: main.tex
\documentclass{article}

\PassOptionsToPackage{numbers, compress}{natbib}

\usepackage[preprint]{neurips_2021}

\usepackage[utf8]{inputenc} %
\usepackage[T1]{fontenc}    %
\usepackage{hyperref}       %
\usepackage{url}            %
\usepackage{booktabs}       %
\usepackage{amsfonts}       %
\usepackage{nicefrac}       %
\usepackage{microtype}      %
\usepackage{xcolor}         %

\usepackage{graphicx}
\usepackage{subcaption}

\usepackage{bbm}
\usepackage{amsmath}
\usepackage{amsthm}
\usepackage[english]{babel}

\usepackage{multirow}
\usepackage{makecell}

\usepackage{algorithm}
\usepackage{algorithmic}

\usepackage{wrapfig}

\newtheorem{problem}{Problem statement}

\newtheorem{definition}{Definition}
\newtheorem{theorem}{Theorem}

\newtheorem{innercustomthm}{Theorem}
\newenvironment{customthm}[1]
  {\renewcommand\theinnercustomthm{#1}\innercustomthm}
  {\endinnercustomthm}

\def\Ds{\mathcal{D}}

\def\Hs{\mathcal{H}}
\def\Is{\mathcal{I}}

\def\Ts{\mathcal{T}}
\def\Tsh{\hat{\mathcal{T}}}

\def\Xs{\mathcal{X}}
\def\Ys{\mathcal{Y}}

\def\Gs{\mathcal{G}}

\def\Ps{\mathcal{P}}

\def\eh{\hat{e}}

\def\ph{\hat{p}}

\def\gh{\hat{g}}

\def\ph{\hat{p}}
\def\wh{\hat{w}}
\def\yh{\hat{y}}

\def\alphah{\hat{\alpha}}
\def\betah{\hat{\beta}}
\def\gammah{\hat{\gamma}}

\def\ie{\emph{i.e., }}
\def\eg{\emph{e.g., }}

\def\ours{Hoki}

\newcommand\revised[1]{{#1}}

\newcommand\eat[1]{}

\title{Confidence Calibration Using Logits Transformations}

\author{%
  Sooyong Jang
  \\
  Dept. of Computer and Info. Science\\
  PRECISE Center\\
  University of Pennsylvania\\
  Pennsylvania, PA 15213 \\
   \And
  Radoslav Ivanov\\
  Dept. of Computer and Info. Science\\
  PRECISE Center\\
  University of Pennsylvania\\
  Pennsylvania, PA 15213 \\
   \AND
  Insup Lee\\
  Dept. of Computer and Info. Science\\
  PRECISE Center\\
  University of Pennsylvania\\
  Pennsylvania, PA 15213 \\
   \And
  James Weimer\\
 Dept. of Computer and Info. Science\\
  PRECISE Center\\
  University of Pennsylvania\\
  Pennsylvania, PA 15213 \\
}

\begin{document}
\maketitle

\input{sections/abstract}

\input{sections/introduction}
\input{sections/related_work}
\input{sections/problem}
\input{sections/calibration}

\input{sections/trans_selection}

\input{sections/implementation}

\input{sections/experiments}

\input{sections/conclusion}
\bibliographystyle{unsrtnat}
\bibliography{reference}

\setcounter{section}{0}
\setcounter{theorem}{0}
\renewcommand\thesection{\Alph{section}}
\appendix
\input{supplement}

\end{document}

%% file: sections/abstract.tex
\begin{abstract}
As machine learning techniques become widely adopted in new domains, especially in safety-critical systems such as autonomous vehicles, it is crucial to provide accurate output uncertainty estimation.
As a result, many approaches have been proposed to calibrate neural networks to accurately estimate the likelihood of misclassification. 
However, while these methods achieve low calibration error, there is space for further improvement, especially in large-dimensional settings such as ImageNet.
In this paper, we introduce a calibration algorithm, named \ours, that works by applying random transformations to the neural network logits. We provide a sufficient condition for perfect calibration based on the number of label prediction changes observed after applying the transformations.
We perform experiments on multiple datasets and show that the proposed approach generally outperforms state-of-the-art calibration algorithms across multiple datasets and models, especially on the challenging ImageNet dataset.
Finally, \ours{} is scalable as well, as it
requires comparable execution time to that of temperature scaling.

\eat{
As machine learning techniques become widely adopted in new domains, especially in safety-critical systems such as autonomous vehicles, it is crucial to provide accurate output uncertainty estimation.
As a result, many approaches have been proposed to calibrate neural networks and to estimate the likelihood of the network being wrong on a given example. However, while these methods achieve low expected calibration error (ECE), few techniques provide theoretical performance guarantees.
In this paper, we provide a novel calibration algorithm with a theoretical bound on the ECE. Our approach works by transforming the logits and/or inputs and recursively performing calibration leveraging the information from the corresponding change in the neural network output.
We perform experiments on multiple datasets, including ImageNet, and show that the proposed approach greatly outperforms state-of-the-art calibration algorithms.  Additionally, our theoretical bounds on ECE is shown to be tight, within $0.1\%$ of the empirical ECE.
}

\end{abstract}

%% file: sections/introduction.tex
\section{Introduction}
\label{sec:introduction}

\begin{wrapfigure}{r}{0.43\textwidth}
    \vskip -0.28in
    \centering
    \includegraphics[width=0.35\textwidth]{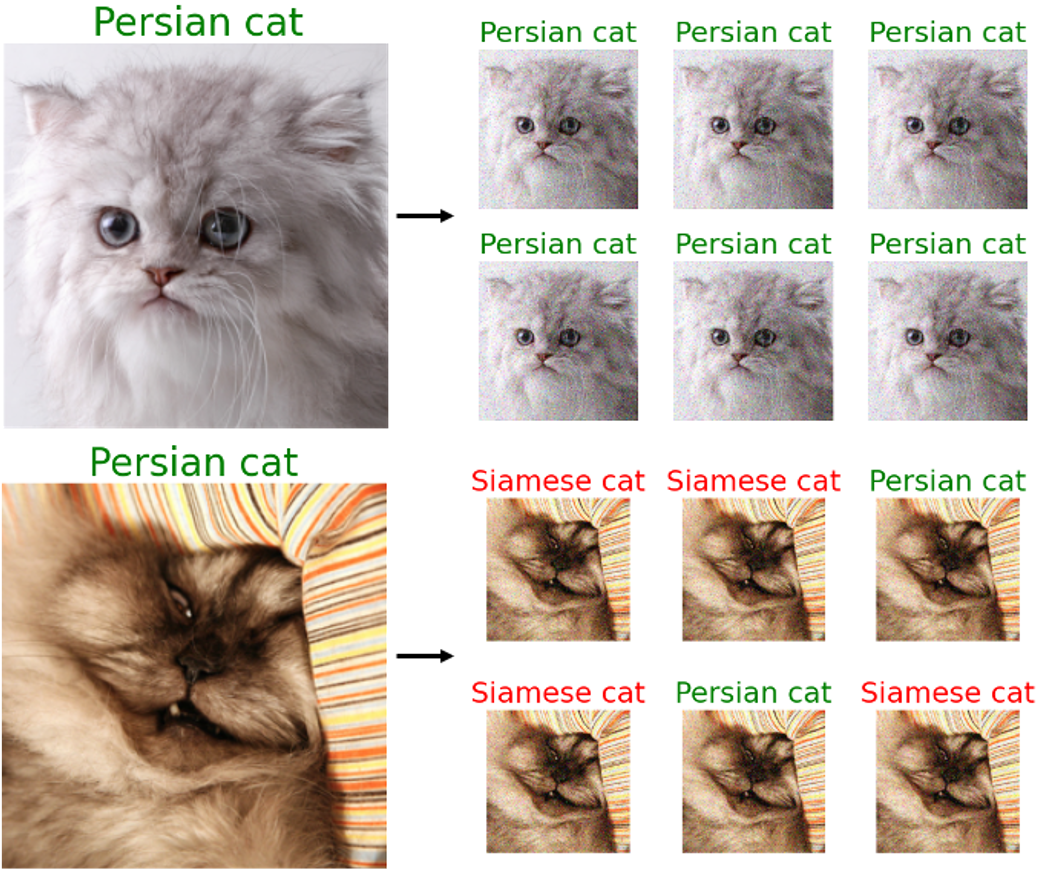}
    \caption{Random noise transformations may lead to a different number of label prediction changes for different images.  Here we apply six different transformations sampled from $N(0, 0.16^2)$.}
    \label{fig:ex_trans}
\end{wrapfigure}

Deep neural networks have proven useful in various fields such as image classification \cite{he2016deep, zagoruyko2016wide, tan2019efficientnet}, object detection \cite{redmon2018yolov3, bochkovskiy2020yolov4}, and speaker verification \cite{li2017deep}.
Motivated by these successes, deep neural networks are now being integrated into safety-critical systems such as medical systems~\cite{ronneberger2015u, arcadu2019deep}.
However, as observed by \citet{guo2017calibration}, neural networks are often over-confident on their predictions. This over-confidence can be a critical problem in the safety-critical applications where over-confident neural networks can be wrong with high confidence.

Various calibration techniques have been proposed to alleviate the miscalibration issue.
A common approach is to map uncalibrated logits or confidences to calibrated ones ~\cite{guo2017calibration, gupta2020calibration, jang2020improving, patel2020multi, platt1999probabilistic, kull2019beyond, rahimi2020intra, wenger2020non}. Another option is to train a neural network with a modified loss function to produce calibrated confidence~\cite{tran2019calibrating, kumar2018trainable, seo2019learning}.
\revised{Although existing methods successfully improve the confidence in terms of calibration metrics such as the expected calibration error (ECE)~\cite{naeini2015obtaining}, there is further space for improvement, especially in high-dimensional settings where vast amounts of data are required for good calibration.}

In this work, we propose \ours{}, a calibration algorithm
that achieves strong empirical performance as compared with state-of-the-art techniques. 
The intuition behind \ours{} is illustrated in Figure~\ref{fig:ex_trans}.
Suppose we are given two images of Persian cats that are initially correctly classified and suppose that we randomly perturb each image a number of times. 
\begin{wrapfigure}{l}{0.5\textwidth}
    \vskip 0.2in
    \centering
    \resizebox{0.5\textwidth}{!}{
    \includegraphics{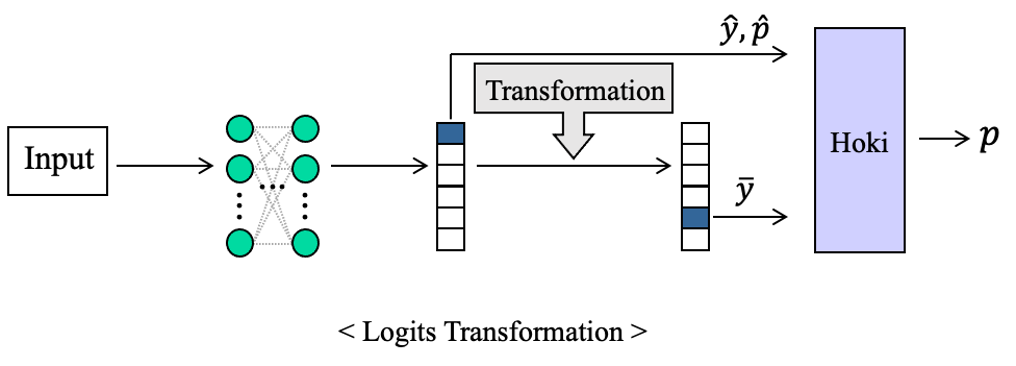}
    }
    \caption{Proposed calibration method. We use logit transformations and corresponding label changes to calibrate the original confidence.}
    \label{fig:overall_process}
\end{wrapfigure}
If all perturbed versions of the first image are still correctly classified whereas the second image transformations lead to label switches (\eg{} to Siamese cat), then intuitively we would have higher confidence in the first image's label.
More generally, the idea, as shown in Figure~\ref{fig:overall_process}, is to use logit transformations to test the network's sensitivity and group examples according to the proportion of observed label switches.

Based on this intuition, we first present a sufficient condition for perfect calibration in terms of ECE through leveraging the information from changes after transformations.
An added benefit of this theoretical result is that it also leads to a natural implementation, through minimizing the empirical calibration error (including transformations). The proposed algorithm is also efficient in terms of runtime, especially in the case of logit transformations due to their reduced dimensionality (as compared to the input space dimensionality).

To evaluate \ours{}, we perform experiments on MNIST, CIFAR10/100, and ImageNet, and we use several standard models per dataset, including LeNet5, DesneNet, ResNet, and ResNet SD.
On these datasets, we compare \ours{} with multiple state-of-the-art calibration algorithms, namely temperature scaling~\cite{guo2017calibration}, MS-ODIR, Dir-ODIR~\cite{kull2019beyond}, 
ETS, IRM, IROvA-TS~\cite{zhang2020mix}, and ReCal~\cite{jang2020improving}.
In terms of ECE~\cite{naeini2015obtaining}, \ours{} outperforms other calibration algorithms on \revised{8} out of 15 benchmarks; we emphasize that \ours{} achieves 
lower ECE than other methods on all ImageNet models. Finally, in terms of learning time, \ours{} achieves similar performance to temperature scaling, the fastest algorithm we compared with.

The main contributions of this paper are as follows:

\begin{itemize}
    \item we propose \ours, an iterative calibration algorithm using logit transformations;
    \item we provide a sufficient condition for perfect calibration;
    \item we show that \ours\ outperforms other calibration algorithms on 8
    out of 15 benchmarks, and achieves the lowest ECE on all ImageNet models;
    \item we demonstrate that \ours{} is also efficient in terms of learning time, achieving performance comparable to temperature scaling.
\end{itemize}

This paper is structured as follows.
In Section \ref{sec:related_work}, we summarize the related works, and in Section \ref{sec:problem_statement}, we describe the problem statement.
Then, we present the theory of calibration through transformations in Section \ref{sec:cal}. 
In Section \ref{sec:trans_selection}, we demonstrate the transformation selection process, and we illustrate the calibration algorithm using transformations in Section \ref{sec:implementation}.
In Section \ref{sec:experiments}, we show the experimental results, and we conclude this paper in Section \ref{sec:conclusion}.

%% file: sections/related_work.tex
\section{Related work}
\label{sec:related_work}

Various approaches have been tried for obtaining accurate confidences, and in this paper, we consider the papers most related to our approach.
We review post-hoc calibration techniques, transformation based calibration, and also calibration with a theoretical guarantee.

\textbf{Post-hoc Calibration.}
Many researchers have proposed calibration methods for a neural network classifier so that the predicted probabilities match the empirical probabilities using a validation set \cite{gupta2020calibration, patel2020multi, platt1999probabilistic, guo2017calibration, kull2019beyond, rahimi2020intra, tran2019calibrating, zhang2020mix}. 
\revised{These post-hoc approaches learn a (simpler) mapping function for the calibration without re-training the given (complex) classifier, and different functions for the different types of inputs have been proposed for the mapping.
For example, temperature scaling \citep{guo2017calibration} uses a linear function on logits, Platt scaling \citep{platt1999probabilistic} employs an affine function on logits, and Dirichlet calibration \citep{kull2019beyond} trains an affine function on log logits.}

\textbf{Transformation-based calibration.}
Approaches using input transformations for calibration have also been proposed.
\citet{bahat2020classification} apply semantic preserving transformations such as contrast change, rotation and zoom to augment given inputs and compute confidence using the augmented set of inputs.
\citet{jang2020improving} introduce a lossy label-invariant transformation for calibration. They define a lossy label-invariant transformation and use it to group inputs and apply group-wise temperature scalings.
While our method is also based on transformations, the choice of random transformations allows us to obtain a theoretical guarantee for perfect calibration.

\textbf{Calibration with theoretical guarantee. }
\citet{kumar2019verified} propose a scaling-binning calibrator which combines Platt Scaling and histogram binning, and provide a bound on the calibration error.
\citet{park2020pac} propose a calibrated prediction method which provides per-prediction confidence bound using Clopper-Pearson interval based on histogram binning.
With this method, examples in the same bin are assigned the same confidence, whereas \ours{} assigns a different confidence to each example, while also aiming to achieve the sufficient condition for perfect calibration within the bin.

%% file: sections/problem.tex
\section{Problem Statement}
\label{sec:problem_statement}

Let $\Xs$ be a feature space, $\Ys = \{1, \dots, C\}$ be a set of labels, and $\Ds$ be a distribution over $\Xs \times \Ys$.
We are given a classifier $f: \Xs  \to \Ys$ and a corresponding calibrator $g: \Xs \to [0,1]$ such that, for a given example $x$, $(f(x), g(x)) = (\yh,\ph)$ is the label prediction $\yh$ with a corresponding confidence $\ph$. In what follows, we say that the sets $\Ps_1, \dots, \Ps_J$ form a bin partition of the confidence space, $[0,1]$, if $\cup_{i=1}^J \Ps_i = [0,1]$ and $\forall i \neq j, \; \Ps_i \cap \Ps_j = \emptyset$. Furthermore, given a dataset $Z = \{(x_1, y_1), \dots, (x_N, y_N)\}$, we say $g$ induces an index partition $\Is_1, \dots, \Is_J$ of $\{1,\dots,N\}$ such that $g(x_n) \in \Ps_j \iff n \in \Is_j, \forall (x_n,y_n) \in Z$.

Before formally stating the problem considered in this work, we define calibration error (CE) and expected calibration error (ECE) \cite{naeini2015obtaining}.
\begin{definition}[Calibration Error (CE)]
\label{def:ce}
For any calibrator $g$ and confidence partitions $\Ps_1, \dots, \Ps_J$, the calibration error (CE) is defined as
\begin{equation*}
CE(g) = \sum_{j = 1}^{J}  w_j  \left| e_j \right|,
\end{equation*}
where
\begin{align*}
& \; e_{j} := P_{\Ds}\left[Y = f(X)\; \middle| \; g(X) \in \Ps_j \right] - E_{\Ds}\left[g(X) \mid g(X) \in \Ps_j \right]\\
& \; w_{j} := P_{\Ds}\left[g(X) \in \Ps_j\right].
\end{align*}
\end{definition}
Intuitively, the CE of a classifier-calibrator pair in a given partition is the expected difference between the classifier's accuracy and the calibrator's confidence. To get the CE over the entire space, we sum up all the individual partition CEs, weighted by the probability mass of each partition (i.e., the probability of an example falling in that partition).
\begin{definition}[Expected Calibration Error (ECE)]
\label{def:ece}
For any calibrator $g$, confidence partitions $\Ps_1, \dots, \Ps_J$, sampled dataset $Z \in (\Xs \times \Ys)^N$, and induced index partition $\{ \Is_1, \dots , \Is_J \}$,
we define the expected calibration error (ECE) as
\begin{equation*}
ECE(g) = \sum_{j = 1}^{J} \wh_j \left| \eh_j \right|,
\qquad
\text{where}\qquad
 \eh_j := \sum_{n \in \Is_j} \frac{\mathbbm{1}_{\{y_n = \yh_n\}} - g(x_n) }{| \Is_j| } \quad \mbox{and} \quad \wh_j := \frac{|\Is_j|}{N}.
\end{equation*}

\end{definition}
Thus, the ECE is the sampled version of the CE. Note that Definition~\ref{def:ece} is equivalent to the standard ECE definition, as used in prior work~\cite{guo2017calibration}. We are now ready to state the problem addressed in this work, namely find a calibrator $\gh$ that minimizes the ECE over a validation set.
\begin{problem}
\label{ps:main_problem}
Let $\Gs = \{g : \Xs \to [0,1]\}$ be the set of all calibrators. We aim to find $\gh \in \Gs$ that minimizes the expected calibration error,
\begin{equation*}
\gh = \arg \min_{g \in \Gs} ECE(g). \\
\end{equation*}

\end{problem}

%% file: sections/calibration.tex
\section{Calibration using Transformations}
\label{sec:cal}

This section provides the intuition and theory of using transformations for the purpose of calibration. 
We begin by providing high level intuition, followed by a sufficient condition for perfect calibration in expectation, which leads to a natural implementation as well.

\paragraph{High-level intuition.} Suppose that $f$ and $g$ form a classifier-calibrator pair. If we take a correctly classified image of a cat, for example, we would expect that the classification confidence would drop as we apply random transformations to the image (e.g., add noise, zoom out). Conversely, if the confidence does not decrease, we would conclude that $f$ and $g$ are not properly calibrated.

More generally, the goal of applying transformations is to group examples in bins of similar confidence. In particular, if a certain set of examples exhibits similar transformation patterns (e.g., label switching, misclassification), then the calibrator should learn to assign such examples a similar confidence value. Of course, this approach would only work for a good choice for transformations -- we discuss a number of options in Section~\ref{sec:trans_selection}.

\paragraph{Sufficiency for perfect calibration.}
We now investigate calibrator properties that ensure perfect calibration. Suppose we are given a class of transformations $\Ts = \{t: \Xs \to \Xs\}$, e.g., functions that add random noise, and a corresponding probability distribution $\Ds_T$ over $\Ts$. Then, for each example $(x,y)$, we can apply a number of transformations and observe how many transformations lead to a label switch. Specifically, the following result is key to achieving perfect calibration.

\begin{theorem}[Sufficiency for Perfect Calibration]
\label{thm:sufficiency}
Let $\Ps_1, \dots, \Ps_J$ be a confidence bin partition. A calibrator $g \in \Gs$ is perfectly calibrated, i.e., $CE(g) = 0$, if it satisfies, $\forall j \in \{1, \dots, J\}$,
\begin{align*}
E_{\Ds}\left[ g(X) \mid g(X) \in \Ps_{j} \right] =  \alpha_{j} \gamma_j + \beta_{j} (1 - \gamma_j) &\\
\end{align*}
where
\begin{align*}
\alpha_{j} =& P_{\Ds \times \Ds_T}[f(X) = Y \mid f(T(X)) = f(X), g(X) \in \Ps_j ] \\
\beta_{j} =& P_{\Ds \times \Ds_T}[f(X) = Y \mid f(T(X)) \neq f(X), g(X) \in \Ps_j ] \\
\gamma_j =& P_{\Ds \times \Ds_T}\left[f(T(X)) = f(X) \mid g(X) \in \Ps_{j} \right] &\\
\end{align*}
\end{theorem}
\begin{proof}
Proof provided in the supplementary material.
\end{proof}

Intuitively, Theorem~\ref{thm:sufficiency} states that the label switching that we observe (in each bin) due to added transformations must be consistent with the confidence and accuracy in that bin. In particular, the average confidence in the bin must be equal to the weighted sum of accuracies over the two groups of examples: 1) examples whose label is changed by some transformation; 2) examples whose label is not changed due to transformations. The benefit of Theorem~\ref{thm:sufficiency} is that it leads to a natural implementation by estimating all probabilities given a validation set (as discussed in Section~\ref{sec:implementation}). In the Supplementary Material, we also provide a theoretical bound (Theorem 2) on the generalization ECE (given new data) of \ours{} in a probably approximately correct sense.

%% file: sections/trans_selection.tex
\section{Transformation Selection}
\label{sec:trans_selection}
As discussed in Section~\ref{sec:cal}, the choice of transformations greatly affects the benefit of the result presented in Theorem~\ref{thm:sufficiency}. In particular, if a certain transformation results in a label switch for all examples, then it does not provide any useful confidence information. Thus, the most beneficial transformations are those that separate different examples into different partitions, as measured by the proportion of label switches caused by those transformations.

\paragraph{Choosing a class of transformations.} The first consideration when selecting a transformation is whether to apply it to the input $x$ or to some internal classifier representation, \eg the logits in last layer of a neural network. The benefits of applying input transformations are that they are independent of the classifier and can be chosen based on physical characteristics (\eg a small rotation should not affect an image's class). Applying transformations to the logits is also appealing due to the reduced dimensionality: this results in improved scalability and makes it easier to find useful transformations.%

Another consideration when choosing the transformations is what family to select them from. Input transformations offer a wide range of possibilities, especially in the case of images, e.g., rotation, translation, zoom out. On the other hand, logit transformations do not necessarily have a physical interpretation, so a more natural choice is to add noise selected from a known probability distribution, \eg Gaussian or uniform. In this paper, we explore the space of uniform noise \revised{and Gaussian noise}, as applied to the neural network's logits, in order to benefit from the scalability improvements due to logit transformations.
\paragraph{Parameter selection.} \revised{As discussed above, the noise parameters need to be chosen so as to maximize the benefit of using transformations. One way of measuring the effect of a given transformation is by computing the standard deviations of (non-calibrated) confidences over the entire validation set. Intuitively, if a transformation results in a large standard deviation of confidences, that means this transformation is correlated with the classifier's sensitivity to input perturbations and hence with the confidence in the classifier's correctness. 
Therefore, we aim to identify transformations that maximize the standard deviation of confidences over the validation data.}

\revised{To compute the variance in predicted confidences for a specific transformation distribution $\Ds_T$, we use the sufficient condition presented in Theorem~\ref{thm:sufficiency}. In particular, suppose we are given a validation set $Z = (x_1, y_1), \dots, (x_N, y_N)$ and a sampled set of transformations $T = \{t_1, \dots, t_M\} \sim [\Ds_T]^M$.
Let $\Is_1, \dots, \Is_J$ be an index partition.\footnote{Note that, when choosing the noise parameters, we use a single bin for all data. Equations~\eqref{eq:alphah}-\eqref{eq:ph} are written for an arbitrary partition since they are referenced in Section~\ref{sec:implementation} as well.} Then, estimates of $\alpha$, $\beta$, $\gamma$, and calibrated confidence, $p$, can be calculated as}
\begin{align}
\label{eq:alphah}
\alphah_{j} =&  \frac{ \displaystyle\sum_{n \in \Is_j} \sum_{m=1}^M \mathbbm{1}_{\{f(x_n) = y_n\}}  \mathbbm{1}_{\{f(t_m(x_n)) = f(x_n)\}} }{ \displaystyle \sum_{n \in \Is_j} \sum_{m=1}^M \mathbbm{1}_{\{f(t_m(x_n)) = f(x_n)\}} } \\
\label{eq:betah}
\betah_{j} =& \frac{ \displaystyle \sum_{n \in \Is_j} \sum_{m=1}^M \mathbbm{1}_{\{f(x_n) = y_n\}}  \mathbbm{1}_{\{f(t_m(x_n)) \neq f(x_n)\}} }{ \displaystyle \sum_{n \in \Is_j} \sum_{m=1}^M \mathbbm{1}_{\{f(t_m(x_n)) \neq f(x_n)\}} } \\
\gammah_{j,n} =& \frac{1}{M}\sum_{m=1}^M \mathbbm{1}_{ \{f(t_m(x_n)=f(x_n) \} }\\
\label{eq:ph} 
\ph_{j,n} =& (\alphah_{j} - \betah_{j}) \gammah_{j,n} + \betah_{j}.
\end{align}

\revised{
To choose the transformation for each dataset-model combination (please refer to Section~\ref{sec:experiments} for a full description of the datasets), we perform grid search over Gaussian and uniform noise parameters and choose the setting that results in the largest standard deviation of $\ph_{j,n}$. In the Gaussian case, we search over the space [-20, 20] for the mean and (0, 20] for the standard deviation.
For uniform noise, we explore the space [-20, 20], by varying both the minimum noise as well as the range of the noise.}

\revised{
Table \ref{tab:selected_transformation} shows the selected transformation parameters and corresponding values for $\alphah$, $\betah$, and standard deviation of $\ph_{j,n}$, denoted by $\hat{\sigma}$, as computed over the different datasets and models used in our experiments.
As shown in the table, $\hat{\sigma}$ varies between 0.0395 and 0.0967, which illustrates the challenge of finding an appropriate transformation. 
}
\begin{table}[!t]
    \centering
    \caption{Selected transformation parameters over the different datasets and models. The number of transformation is $M=1000$. We use $U(a,b)$ to denote uniform noise with a range of $[a, b]$ and $G(a,b)$ to denote Gaussian noise with a mean of $a$ and standard deviation of $b$.}
    \label{tab:selected_transformation}
    \vskip 0.15in
    \begin{small}
    \begin{sc}
    \resizebox{0.5\textwidth}{!}{
    \begin{tabular}{llcccc}
    \toprule
    Dataset & Model& Parameters & $\alphah$ & $\betah$ & $\hat{\sigma}$\\
    \midrule
    MNIST & LeNet 5 & U(-2, 4) & 0.9910 & 0.6358  & 0.0167\\
    \midrule
    \multirow{5}{*}{CIFAR10} & DenseNet 40 & $U(5, 14)$ & 0.9399 & 0.6046 & 0.0456 \\
     & LeNet 5 & $U(16, 19)$ & 0.7795 & 0.4826 & 0.0616\\
     & ResNet 110 & $U(-6, 3)$ & 0.9567 & 0.6320  & 0.0395 \\ 
     & ResNet 110 SD & $U(-16, -8)$ & 0.9286 & 0.5990 & 0.0500 \\
     & WRN 28-10 & $U(-16, -8)$ & 0.9715 & 0.6529 & 0.0332\\ 
    \midrule
    \multirow{5}{*}{CIFAR100} & DenseNet 40 & $G(-20, 2)$ & 0.7615 & 0.4082 & 0.0800\\
     & LeNet 5 & $G(-5, 1)$ & 0.4817 & 0.2542 & 0.0553\\
     & ResNet 110 & $G(-4, 2)$ & 0.7837 & 0.4457 & 0.0801\\ 
     & ResNet 110 SD & $G(-19, 2)$ & 0.7870 & 0.4558 & 0.0782\\ 
     & WRN 28-10 & $G(4, 2)$ & 0.8706 & 0.5038 & 0.0967\\ 
    \midrule
    \multirow{4}{*}{ImageNet} & DenseNet 161 & $G(16, 2)$ & 0.8464 & 0.5190 & 0.0806\\
     & MobileNet V2 & $G(3, 2)$ & 0.8222 & 0.5013 & 0.0871 \\
     & ResNet 152 & $G(0, 2)$ & 0.8551 & 0.5268 & 0.0803  \\ 
     & WRN 101-2 & $G(3, 2)$ & 0.8596 & 0.5236 & 0.0820\\ 
    \bottomrule
    \end{tabular}
    }
    \end{sc}
    \end{small}
\end{table}

%% file: sections/implementation.tex
\section{Implementation}
\label{sec:implementation}

Based on the theory described in Section \ref{sec:cal}, we propose \ours{}, an iterative algorithm for confidence calibration.
\ours{} operates differently during design time and runtime.
During design time, \ours{} samples random transformations and learns the $\alphah_j$ and $\betah_j$ parameters for each bin. These parameters are then used at runtime, on test data, to estimate the confidence for new examples.

\paragraph{Design Time Algorithm.}
The design time algorithm is described in Algorithm~\ref{alg:design_time}. %
The high-level idea of the algorithm is to achieve the sufficient condition outlined in Theorem~\ref{thm:sufficiency} on the validation set. In particular, we first sample transformations $\{t_1, \dots, t_M\}$ from $\Ts$ and observe the corresponding predictions $\bar{y}_{n,m}$. Based on $\bar{y}_{n,m}$, we compute the fraction of transformed examples that have the same label as the original image, $\gamma_n$. We emphasize that, as a special case of Theorem~\ref{thm:sufficiency}, $\gamma_n$ is computed separately for each example, as opposed to averaged over the entire partition. This modification ensures that the data is spread across multiple partitions, while still satisfying the sufficient condition for perfect calibration.

After the initialization step, \ours{} recursively estimates the parameters $\alpha_j$ and $\beta_j$ and computes the calibrated confidence using those two values based on Theorem \ref{thm:sufficiency}. Since the original partitioning may change after reestimating the parameters, we repeat this process until there is no change in the data partitioning (or a maximum number of iterations is reached). Note that for the corner case where the transformations result in empty sets (i.e., either all labels change or all remain the same), we set $\alphah_j^k$, $\betah_j^k$ to the bin accuracy.  Ultimately, at design time, \ours{} returns the set of calibration pairs for all partitions for all iterations learned in the design time algorithm as a set $\Hs$, the set of transformations $\Tsh$, and the validation data accuracy, $p$.

\begin{algorithm}[tb]
   \caption{Design Time Algorithm}
   \label{alg:design_time}
\begin{algorithmic}
   \STATE {\bfseries Input:} validation set $Z = (x_1, y_1), \dots, (x_N, y_N)$, transformation set $\Ts$, number of transformations $M$, classifier $f$, confidence space partition $\Ps_1, \dots, \Ps_J$, maximum number of iterations $K$
   \STATE
   \STATE Sample transformations $\Tsh = \{t_1, \dots, t_M\}$ from $\Ts$
   \STATE $\ph = \frac{1}{N} \sum_{n = 1}^N \mathbbm{1}_{\{y_n = f(x_n)\}}$
   \FOR{$n=1$ {\bfseries to} $N$} 
   \STATE $\gammah_n = \frac{1}{M} \sum_{m = 1}^M \mathbbm{1}_{\{f(t_m(x_n)) = f(x_n)\}}$
   \STATE $p_n = \ph$
   \ENDFOR
   \STATE 
   \FOR{$k=1$ {\bfseries to} $K$}
   \STATE Compute sets $\Is_1^k, \dots, \Is_J^k$ s. t. $n \in \Is_j^k \iff p_n \in \Ps_j$
   \IF{$k > 1 \wedge \Is_1^k = \Is_1^{k-1} \wedge \dots \wedge \Is_J^k = \Is_J^{k-1}$}
   \STATE \textbf{break} 
   \ENDIF
   \FOR{$j=1$ {\bfseries to} $J$}
   \STATE $c =\sum_{n \in \Is_j^k} \sum_{m=1}^M \mathbbm{1}_{\{f(x_n) = f(t_m(x_n))\}} $
   \IF {$c = 0 \vee c = M|\Is_j^k|$}
   \STATE $\alphah^k_j = \betah^k_j =  \frac{1}{|\Is_j^k|} \sum_{n \in \Is_j^k} \mathbbm{1}_{\{y_n = f(x_n)\}}$
   \ELSE
   \STATE Compute $\alphah_j^k$ according to Equation~\eqref{eq:alphah}, using $\Is_j^k$
   \STATE Compute $\betah_j^k$ according to Equation~\eqref{eq:betah}, using $\Is_j^k$
   \ENDIF
   \STATE $p_n = \left(\alphah^k_j - \betah^k_j \right) \gammah_n + \betah^k_j, \; \forall n \in \Is_j^k$
   \ENDFOR
   \STATE $K^* = k$
   \ENDFOR
   \STATE $\Hs = \{ (\alphah_j^k, \betah_j^k) \; |  \; 1 \leq j \leq J, \; 1 \leq k \leq K^*\}$
   \STATE return $\Hs$, $\ph$, $\Tsh$
\end{algorithmic}
\end{algorithm}

\paragraph{Runtime Algorithm.}
The runtime algorithm is described in Algorithm~\ref{alg:run_time}. %
Once the parameters for each step, $\alphah_j^k$ and $\betah_j^k$, are learned in Algorithm~\ref{alg:design_time}, \ours{} can calibrate the confidence for a new input $x$ using the calibration parameter pairs in $\Hs$.
\ours{} first observes the original prediction $f(x)$ and the transformed data prediction $f(t_m(x))$ for all transformations $\{t_1, t_2, \dots, t_M\}$.
\ours{} computes $\gamma$ based on these values and iteratively updates the confidence according to the calibration parameters learned in Algorithm~\ref{alg:design_time}. 
\begin{algorithm}[tb]
   \caption{Runtime Algorithm}
   \label{alg:run_time}
\begin{algorithmic}
   \STATE {\bfseries Input:} test sample $x \in \Xs$, original classifier $f$, outputs of Algorithm~\ref{alg:design_time}: $\Hs$, $\ph$, $\Tsh$. 
   \STATE $\gamma= \frac{1}{M} \sum_{m = 1}^M \mathbbm{1}_{\{f(x)=f(t_m(x))\}}, \; t_m \in \Tsh$
   \STATE $p = \ph$
   \FOR{$k=1$ {\bfseries to} $\frac{|\Hs|}{J}$}
   \STATE Identify partition index $j'$ for $x$ such that $p \in \Ps_{j'}$
   \STATE Identify calibration parameters $(\alphah^k_{j'}, \betah^k_{j'}) \in \Hs$
   \STATE $p=\left(\alpha^k_{j'} - \beta^k_{j'} \right) \gamma + \beta^k_{j'}$
   \ENDFOR
   \STATE return $p$
\end{algorithmic}
\end{algorithm}

\paragraph{Limitations.}
\label{sec:limitations}
There are two main limitations of our approach. First of all, since \ours{} uses transformations for calibration, choosing the appropriate transformations has a significant effect on performance.
To address this issue, we propose a transformation selection based on the maximization of confidence variance as described in Section \ref{sec:trans_selection}. The second limitation is the utility of the ECE metric itself -- it is possible to minimize the ECE by outputting the network accuracy as confidence for all examples, which defeats the purpose of calibration. We argue that by maximizing the variance of our algorithm, we ensure that the data is spread across multiple bins, as demonstrated in Section~\ref{sec:experiments}. Thus, one can use the ECE in multiple ways, e.g., in autonomous systems by making an informed decision through taking into account the calibrated confidence in a new example's predicted label.

%% file: sections/experiments.tex
\section{Experiments}
\label{sec:experiments}
We compare \ours{} with state-of-the-art calibration algorithms using several standard datasets and models.
For each model and dataset, we compute the ECE for the uncalibrated model and the calibrated confidence by the  algorithms.
The experimental setup, baseline algorithms and the evaluation metrics are explained in the following subsections.

\subsection{Experimental Setup}
This subsection provides the details about the datasets, models, baseline algorithms, and the evaluation metrics.

\textbf{Datasets and Models.}
\revised{We perform experiments on MNIST \cite{lecun1998gradient}, CIFAR 10/100 \cite{krizhevsky2009learning}, and ImageNet \cite{deng2009imagenet}}.
We use the following models for each dataset. 
For MNIST, we use one model, LeNet5 \cite{lecun1998gradient}.
For CIFAR 10/100, we use five different models, DenseNet 40 \cite{huang2017densely}, LeNet5 \cite{lecun1998gradient}, ResNet110 \cite{he2016deep}, ResNet110 SD \cite{huang2016deep}, and WRN-28-10 \cite{zagoruyko2016wide}.
For ImageNet, we use four models,  DenseNet161 \citep{huang2017densely}, MobileNetV2 \cite{sandler2018mobilenetv2}, ResNet152 \citep{he2016deep}, and WRN-101-2 \cite{zagoruyko2016wide}.

We implement LeNet5, ResNet110 SD, and obtain code for DenseNet40\footnote{https://github.com/andreasveit/densenet-pytorch under BSD 3-Clause License}, ResNet110\footnote{\label{fn:resnet110github}https://github.com/bearpaw/pytorch-classification, under MIT License} and WRN 28-10\footnotemark[\getrefnumber{fn:resnet110github}] from the corresponding github repositories.
We also obtained the pre-trained model for all models on ImageNet from PyTorch.\footnote{https://pytorch.org/docs/stable/torchvision/models.html, under BSD License}

\textbf{Baselines.}
We compare \ours{} with several state-of-the-art calibration algorithms, namely temperature scaling, vector scaling \cite{guo2017calibration}, Dir-ODIR, MS-ODIR \cite{kull2019beyond}, \revised{ETS,  IRM, IROvA-TS \cite{zhang2020mix}}, and ReCal \cite{jang2020improving}.
They calibrate confidence by learning a mapping function for uncalibrated logits or confidences. 
We obtain other calibration algorithms from their papers except for temperature scaling and vector scaling which we obtain from \citet{kull2019beyond}.
For ReCal, the authors provide three different setups for their algorithm, and we choose ('zoom-out', 0.1, 0.9, 20) because it shows the best results on ImageNet.

\textbf{Evaluation Metric.}
As described in Problem Statement~\ref{ps:main_problem}, we evaluate all algorithms based on ECE (with $J=15$ bins of equal width~\cite{guo2017calibration, nixon2019measuring}), as defined in Definition \ref{def:ece}.
Additionally, we calculate the learning time during design time to investigate each algoritm's practical utility.
If a calibration algorithm is too slow on real datasets, it may not be appropriate to use the algorithm in practice. 

\subsection{Results}
The experimental results have two parts.
The first part is a comparison on calibration performance in terms of ECE, and the second part is an analysis on time efficiency during design time.

\paragraph{ECE Results.}

\begin{table*}[bt]
\caption{ECE for the different calibration algorithms on different datasets and models. The number with the bold face and the underline denote the best and the second best result, respectively.}
\label{tab:exp_res_ECE}
\begin{center}
\resizebox{\textwidth}{!}{
\begin{sc}
\begin{tabular}{llcccccccccccc}
\toprule
Dataset & Model & \thead{Val Acc.\\(\%)} & \thead{Test Acc.\\(\%)} & Uncal. & TempS & VecS & \thead{MS-\\ODIR} & \thead{Dir-\\ODIR} & ETS & IRM & \thead{IROvA-\\TS} & ReCal & \ours\\
\midrule
MNIST & LeNet5 & 98.85 & 98.81 & 0.0076 & 0.0018 & \underline{0.0015} & 0.0024 & 0.0022 & 0.0019 & 0.0019 & 0.0033 & 0.0021 & \textbf{0.0008}\\
\midrule
\multirow{5}{*}{CIFAR10} & DenseNet 40 & 91.92 & 91.75 & 0.0520 & 0.0070 & \underline{0.0044} & 0.0052 & \textbf{0.0039} & 0.0069 & 0.0095 & 0.0107 & 0.0101 & 0.0057 \\
 & LeNet5 & 72.00 & 72.77 & 0.0182 & 0.0120 & \textbf{0.0092} & 0.0141 & \underline{0.0105} & 0.0115 & 0.0167 & 0.0229 & 0.0118 & 0.0110 \\
 & ResNet110 & 94.12 & 93.10 & 0.0456 & 0.0088 & 0.0094 & 0.0088 & 0.0084 & \textbf{0.0066} & 0.0103 & 0.0133 & 0.0090 & \underline{0.0071}\\
 & ResNet110 SD & 90.28 & 90.38 & 0.0538 & 0.0114 & \underline{0.0086} & 0.0102 & 0.0094 & 0.0112 & 0.0113 & 0.0156 & 0.0120 & \textbf{0.0044}\\
 & WRN 28-10 & 96.06 & 95.94 & 0.0251 & 0.0097 & 0.0096 & 0.0092 & 0.0094 & 0.0157 & \underline{0.0049} & 0.0088 & 0.0091 & \textbf{0.0026} \\
\midrule
\multirow{5}{*}{CIFAR100} & DenseNet 40 & 68.82 & 68.16 & 0.1728 & 0.0154 & 0.0266 & 0.0296 & 0.0189 & 0.0136 & \underline{0.0135} & 0.0377 & 0.0154 & \textbf{0.0073}\\
 & LeNet5 & 37.82 & 37.66 & \textbf{0.0100} & 0.0211 & 0.0155 & 0.0131 & 0.0142 & \underline{0.0120} & 0.0125 & 0.0363 & 0.0192 & 0.0123\\
 & ResNet 110 & 70.60 & 69.52 & 0.1422 &\textbf{0.0091} & 0.0300 & 0.0345 & 0.0231 & 0.0155 & 0.0202 & 0.0457 & \underline{0.0121} & 0.0127\\
 & ResNet 110 SD & 70.62 & 70.10 & 0.1229 & \underline{0.0089} & 0.0358 & 0.0355 & 0.0207 & \textbf{0.0086} & 0.0142 & 0.0425 & 0.0100 & 0.0098\\
 & WRN 28-10 & 79.62 & 79.90 & 0.0534 & 0.0437 & 0.0452 & 0.0355 & 0.0346 & 0.0370 & \textbf{0.0108} & 0.0336 & 0.0373 & \underline{0.0112}\\
\midrule
\multirow{2}{*}{ImageNet} & DenseNet 161 & 76.83 & 77.45 & 0.0564 & 0.0199 & 0.0233 & 0.0368 & 0.0477 & 0.0100 & \underline{0.0090} & 0.0487 & 0.0133 & \textbf{0.0043}\\
 & MobileNet V2 & 71.69 & 72.01 & 0.0274 & 0.0164 & 0.0153 & 0.0212 & 0.0269 & 0.0087 & \underline{0.0075} & 0.0477 & 0.0153 & \textbf{0.0011}\\
 & ResNet 152 & 77.93 & 78.69 & 0.0491 & 0.0201 & 0.0207 & 0.0347 & 0.0397 & 0.0112 & \underline{0.0080} & 0.0457 & 0.0139 & \textbf{0.0052}\\
 & WRN 101-2 & 78.67 & 79.15 & 0.0524 & 0.0307 & 0.0330 & 0.0418 & 0.0279 & 0.0165 & \underline{0.0086} & 0.0426 & 0.0258 & \textbf{0.0067}\\
 \bottomrule
\end{tabular}
\end{sc}
}
\end{center}
\end{table*}

Table~\ref{tab:exp_res_ECE} displays ECE values for each algorithm, along with each model's validation set and test set accuracy (the Supplementary Material provides an extensive evaluation where we also vary the number of bins in the ECE evaluation, in order to test each algorithm's robustness to more fine-grained bins).
As discussed in Section~\ref{sec:trans_selection}, \ours{} uses the transformations shown in Table~\ref{tab:selected_transformation}. As shown in Table~\ref{tab:exp_res_ECE}, \ours{} achieves the lowest ECE on \revised{8}
out of 15 benchmarks. The benefit of using transformations is especially pronounced in the large-dimensional ImageNet dataset where \ours{} consistently achieves the lowest ECE on all models.

Interestingly, the ECE produced by \ours{} closely tracks the difference between the validation and test set accuracy. In some sense, one cannot hope to do better than this difference as it reflects the variance within each dataset. Thus, the benchmarks where \ours{} does not achieve the best performance are settings with large differences in generalization accuracy. For example, in the case of ResNet110 on CIFAR10 and CIFAR100, the gaps are 1.02 and 1.08 percentage points, respectively.

Another reason for our strong performance on ImageNet is the dataset size (there are 25,000 images in the ImageNet validation set compared to 10,000 images in MNIST and 5,000 images in CIFAR10/100).
A larger validation set means that each partition is likely to have more samples, which in turn results in more accurate estimation of $\alpha$, $\beta$, and $\gamma$. %

\begin{figure}
    \centering
    \includegraphics[width=\textwidth]{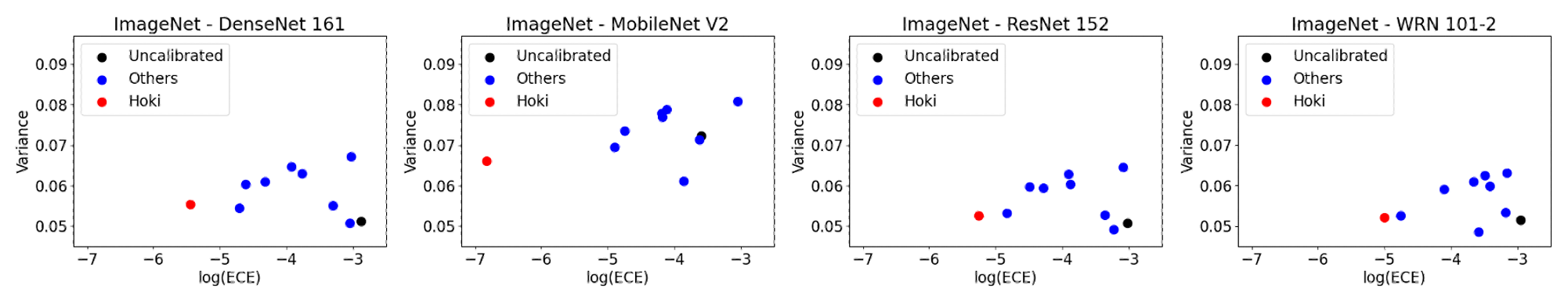}
    \caption{$\log$ ECE vs. Variance on the various ImageNet models.}
    \label{fig:log_ece_var}
\end{figure}

\paragraph{ECE Variance.} For further evaluation, we also explore the variance in confidences produced by each algorithm. As noted in Section~\ref{sec:trans_selection}, a larger variance is preferred because it provides some indication that examples with low true confidence are indeed separated from those with high confidence. Figure~\ref{fig:log_ece_var} provides a plot of the ECE variance vs. the ECE \revised{(in log scale)} for all algorithms on the ImageNet models (plots for the other benchmarks are provided in the Supplementary Material). As shown in the figure, \ours{} has comparable variance to other algorithms that also achieve low ECE on ImageNet. Overall, there appears to be a trade-off between achieving \revised{low} 
ECE and high variance -- we leave exploring this phenomenon for future work.

\paragraph{Time Efficiency.}
Table \ref{tab:learng_time} displays the learning time of each calibration algorithm during design time.
The main reason for the proposed method's good scalability is that we apply transformations to the logits -- thus, we avoid the need to perform the input transformations that are needed in ReCal, for example.
Furthermore, \ours{} is comparable to temperature scaling, which is a fairly simple approach in the sense that it only needs to learn one parameter.
In summary, \ours{} not only achieves low ECE on most benchmarks but is also fast to execute.

\begin{table}[!hb]
    \centering
    \caption{Learning time (sec). The number with the bold face and the underline are the best and the second best result, respectively.}
    \label{tab:learng_time}
    \vskip 0.15in
    \resizebox{1.0\textwidth}{!}{
    \begin{sc}
    \begin{tabular}{llccccccccc}
    \toprule
    Dataset & Model & TempS & VecS & MS-ODIR & Dir-ODIR & ETS & IRM & IROvA-TS & ReCal & \ours\\
    \midrule
    MNIST & LeNet5 & 0.16 & 43.31 & 112.04 & 207.88  & 0.31 & \textbf{0.02} & \underline{0.12} & 57.67 & 0.53\\
    \midrule
\multirow{5}{*}{CIFAR10}& DenseNet40 & 0.08 & 31.33 & 222.93 & 92.34 & \underline{0.04} & \textbf{0.01} & 0.06 & 84.04 & 0.34\\
 & LeNet5 & 0.05 & 11.86 & 79.62 & 74.33 & \underline{0.04} & \textbf{0.01} & 0.06 & 110.79 & 0.32\\
 & ResNet110 & 0.07 & 27.17 & 193.04 & 87.73 & \underline{0.05} & \textbf{0.01} & 0.07 & 38.85 & 0.28\\
 & ResNet110 SD & 0.07 & 21.39 & 189.27 & 93.17 & \underline{0.04} & \textbf{0.01} & 0.09 & 58.74 & 0.28\\
 & WRN 28-10 & 0.05 & 22.71 & 123.46 & 92.80 & \underline{0.04} & \textbf{0.01} & 0.08  & 49.62 & 0.39\\
\midrule
\multirow{5}{*}{CIFAR100} & DenseNet40 & \underline{0.51} & 23.17 & 1211.68 & 626.00 & 0.54 & \textbf{0.11} & 0.57 & 136.23 & 1.09\\
 & LeNet5 & 0.42 & 24.17 & 459.59 & 236.87 & \underline{0.29} & \textbf{0.12} & 0.45 & 97.77 & 0.99\\
 & ResNet110 & \underline{0.30} & 25.49 & 1459.71 & 510.12& 0.59 & \textbf{0.12} & 0.57  & 97.29 & 1.02\\
 & ResNet110 SD & \underline{0.24} & 25.51 & 1696.10 & 495.23 & 0.61 & \textbf{0.12} & 0.63  & 604.12 & 1.01\\
 & WRN 28-10 & \underline{0.30} & 26.71 & 1110.11 & 611.52 & 0.64 & \textbf{0.11} & 0.50  & 125.84 & 1.03\\
\midrule
\multirow{4}{*}{ImageNet} & DenseNet161 & \underline{18.79} & 179.61 & 13901.38 & 6891.19 & 29.34 & \textbf{8.76} & 31.16  & 50730.17 & 37.68\\
 & MobileNet V2 & \underline{18.74} & 423.48 & 3899.07 & 12695.64& 27.56 & \textbf{8.80} & 26.99  & 3139.60 & 32.97\\
 & ResNet152 & \underline{18.79} & 169.72 & 12401.58 & 5402.85  & 29.28 & \textbf{9.01} & 31.14 & 71254.34 & 33.78\\
 & WRN 101-2 & \underline{18.73} & 182.39 & 16989.40 & 11378.40 & 29.17 & \textbf{8.67} & 31.46 & 31545.77 & 34.26\\
\bottomrule
    \end{tabular}
    \end{sc}
    }
\end{table}

%% file: sections/conclusion.tex
\section{Conclusion}
\label{sec:conclusion}

\revised{This work proposed a confidence calibration algorithm based on the intuition that we can partition examples based on the neural network's sensitivity to transformations.}
Based on this intuition, we provided a sufficient condition for perfect calibration in terms of ECE.
We performed an extensive experimental comparison and demonstrated that \ours{} outperforms state-of-the-art approaches in multiple datasets and models, and the benefits are especially pronounced on the challenging ImageNet. For future work, we plan to explore the benefits of combining different transformations, particularly a mix of input and logit transformations.
If those transformations are chosen carefully in order to identify input sensitivity, we expect that more accurate calibration is possible.

%% file: supplement.tex
\onecolumn
\title{Confidence Calibration with Bounded Error Using Transformation:\\ Supplementary Material}

\section{Proof of Theorem 1}
We begin by observing for any $g$ and $\Ps_j \in \{ \Ps_1, \dots, \Ps_j\}$, the law of total probability states
\begin{align*}
 P\left[Y = f(X) \mid g(X)\in \Ps_j \right] = &\\
= & \;P\left[Y = f(X) \mid f(T(X)) = f(X), g(X) \in \Ps_j \right]
  P\left[f(T(X)) = f(X) \mid g(X) \in \Ps_j   \right] \\
&+ P\left[Y = f(X) \mid f(T(X)) \neq f(X), g(X) \in \Ps_j \right]
P\left[f(T(X)) \neq f(X) \mid g(X) \in \Ps_j   \right] \\
  = & \;\alpha_{j} \gamma_j 
  +  \beta_j  P\left[f(T(X)) \neq f(X) \mid g(X) \in \Ps_j   \right] \\
  = & \;\alpha_j \gamma_j +  \beta_j  (1 - \gamma_j)  \\
\end{align*}
Then, from Definition~\ref{def:ece}, %
\begin{align*}
CE(g) = \sum_{j = 1}^{J} w_j \left|e_j \right| 
= & \sum_{j = 1}^{J} w_j \Big| P\left[Y = f(X) \mid g(X)\in \Ps_j \right] - E\left[ g(X) \mid g(X) \in \Ps_j   \right] \Big|\\
= & \sum_{j = 1}^{J} w_j \Big| \alpha_j \gamma_j + \beta_j (1 - \gamma_j) - E\left[ g(X) \mid g(X) \in \Ps_j   \right] \Big| \\
\end{align*}
Thus, the following is a sufficient property for $CE(g) = 0$:
\begin{align*}
\forall j \in \{1, \dots, J\}, \;  E\left[ g(X) \mid g(X) \in \Ps_j   \right]  =  \alpha_j \gamma_j + \beta_j (1 - \gamma_j)
\end{align*}

\section{Generalization bounds on the ECE}
\label{sec:app_thm2}
This section presents a bound on the generalization ECE, given a new dataset, in a probably approximately correct (PAC) sense. Theorem 2 states that if a calibrator $g$ achieves a low ECE on a test set, $Z$, then the expected calibration error of $g$ can be bounded, in a PAC sense.

\begin{customthm}{2}[Bounded Calibration Error]
\label{thm:generalization}
Suppose a calibrator $g$ was evaluated on a test set $Z = \{(x_1, y_1), \dots, (x_N, y_N)\}$, achieving $ECE_Z(g)$.
For any $\delta$, the CE is bounded,~i.e., 
\begin{equation*}
P\left[CE(g) \leq \epsilon \right] \geq 1 - \delta,
\end{equation*}
when
\begin{align*}
	& \epsilon = ECE_Z(g) + \frac{J\sqrt{2}}{\sqrt{N}}  \sqrt{ 2 \ln(2) - \ln(\delta)}
\end{align*}
\end{customthm}

\begin{proof}

\begin{align*}
P\left[CE(g) \geq \epsilon \right] = &\\
= & \; P\left[ \sum_{j =1}^J \left| e_j \right| w_j \geq \epsilon \right] \\
= & \; P\left[ \sum_{j =1}^J \left| e_j \right| (w_j - \wh_j + \wh_j)  \geq \epsilon \right] \\
\leq & \; P\left[ \sum_{j =1}^J \left| e_j \right| \left|w_j - \wh_j \right| + \left| e_j \right| \wh_j  \geq \epsilon \right] \\
\leq & \; P\left[ \sum_{j =1}^J \left|w_j - \wh_j \right| + \left| e_j \right| \wh_j  \geq \epsilon \right] \\
\leq & \; P\left[ \sum_{j =1}^J \left|w_j - \wh_j \right| + \left| e_j - \eh_j \right| \wh_j  + \left| \eh_j \right| \wh_j \geq \epsilon \right] \\
= & \; P\left[ \sum_{j =1}^J \left|w_j - \wh_j \right| + \left| e_j - \eh_j \right|\wh_j \geq \epsilon - ECE_Z(g) \right] \\
\leq & \; \max_j P\left[ \left|w_j - \wh_j \right| + \left| e_j - \eh_j \right| \wh_j \geq \frac{\epsilon - ECE_Z(g)}{J} \right] \\
\leq & \; \max_j  P\left[ \left|w_j - \wh_j \right| \geq \frac{\epsilon - ECE_Z(g)}{2J} \right] + P\left[ \left| e_j - \eh_j \right| \geq \frac{\epsilon - ECE_Z(g)}{2 J  \wh_j} \right] \\
\leq & \; \max_j 2 \exp\left\{ - 2 N  \left(\frac{\epsilon - ECE_Z(g)}{2J} \right)^2 \right\} + 2 \exp\left\{ - 2 N \wh_j \left( \frac{\epsilon - ECE_Z(g)}{2 J \wh_j} \right)^2 \right\} \\
= & \; \max_j 2 \exp\left\{ - 2 N  \left(\frac{\epsilon - ECE_Z(g)}{2 J}\right)^2 \right\} + 2 \exp\left\{ - 2 N \frac{1}{\wh_j}  \left(\frac{\epsilon - ECE_Z(g)}{2 J}  \right)^2 \right\} \\
\leq & \; 2 \exp\left\{ - 2 N \left(\frac{\epsilon - ECE_Z(g)}{2J}\right)^2 \right\} + 2 \exp\left\{ - 2 N  \left(\frac{\epsilon - ECE_Z(g)}{2J}  \right)^2 \right\} \\
= & \; 4 \exp\left\{ - \frac{ N  (\epsilon - ECE_Z(g))^2}{2 J^2} \right\} \\
\end{align*}
We complete the proof by observing
\begin{align*}
& \;  4 \exp\left\{ - \frac{N (\epsilon - ECE_Z(g))^2}{2 J^2} \right\}\leq \delta 
\iff \; \epsilon \geq ECE_Z(g) + \frac{J\sqrt{2}}{\sqrt{N}}  \sqrt{\left( 2 \ln(2) - \ln(\delta) \right) } 
\end{align*}
\end{proof}

\section{Additional Experiments}
In this section, we present additional experimental results.
We show more plots for ECE variance, comparisons using ECE with different number of bins, and Hoki's ECE with the different initialization.

\subsection{ECE Variance}
In addition to the ImageNet result in the main paper, we show the same ECE variance results for other benchmarks on MNIST (Figure~\ref{fig:ece_var_mnist}) and CIFAR10/100 (Figure~\ref{fig:ece_var_cifar10} and \ref{fig:ece_var_cifar100}).
The range of variance are different depending on dataset, but the widths of the range are equal.

Similar to ImageNet case in the main text, Hoki has comparable variance for MNIST (Figure~\ref{fig:ece_var_mnist}) and CIFAR10 (Figure~\ref{fig:ece_var_cifar10}), but with the smaller ECE.
As shown in Figure~\ref{fig:ece_var_cifar100}, Hoki has a similar pattern with ImageNet case on CIFAR100, \ie{}it has comparable variance with better ECE.
Note that uncalibrated classifier is not shown for DenseNet40~(Figure~\ref{fig:ece_var_cifar100_densenet40}), ResNet110~(Figure~\ref{fig:ece_var_cifar100_resnet110}), and ResNet110 SD~(Figure~\ref{fig:ece_var_cifar100_resnet110sd}), because the uncalibrated ECEs are high compared to other algorithms as shown in Table~\ref{tab:exp_res_ECE} and the variances are low.

\begin{figure}[!hbt]
    \centering
\begin{subfigure}{.49\textwidth}
    \includegraphics[width=0.9\textwidth]{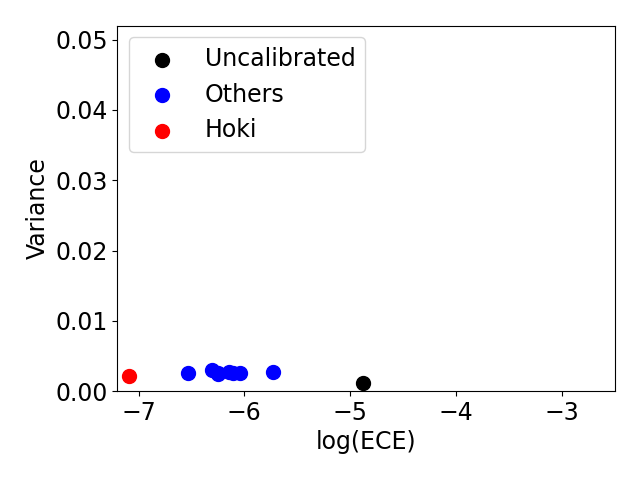}
    \caption{LeNet5}
    \label{fig:ece_var_mnist_lenet5}
\end{subfigure}
\caption{log(ECE) vs. Variance on MNIST.}
\label{fig:ece_var_mnist}
\end{figure}

\begin{figure}[!hbt]
    \centering
\begin{subfigure}{.49\textwidth}
    \includegraphics[width=0.9\textwidth]{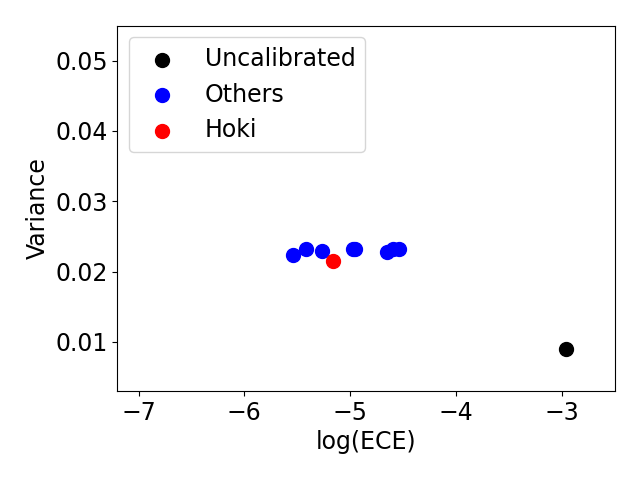}
    \caption{DenseNet40}
    \label{fig:ece_var_cifar10_densenet40}
\end{subfigure}
\begin{subfigure}{.49\textwidth}
    \includegraphics[width=0.9\textwidth]{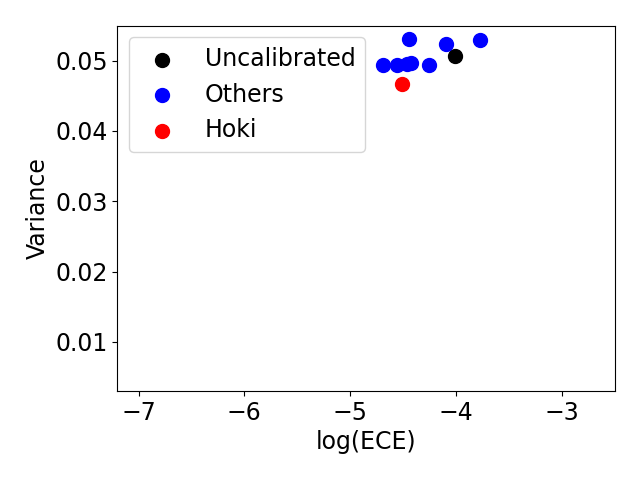}
    \caption{LeNet5}
    \label{fig:ece_var_cifar10_lenet5}
\end{subfigure}
\\
\begin{subfigure}{.49\textwidth}
    \includegraphics[width=0.9\textwidth]{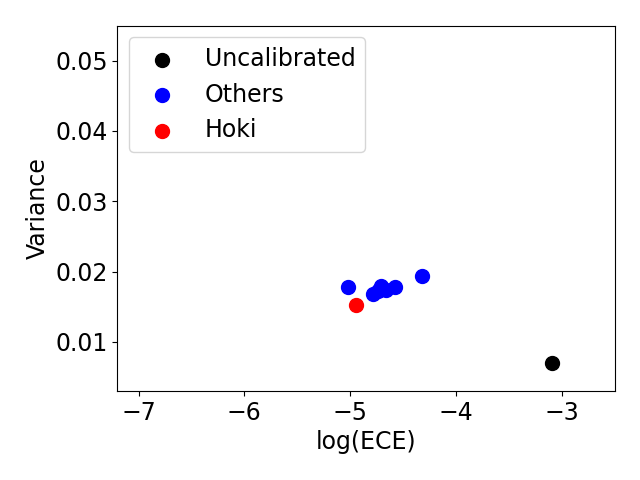}
    \caption{ResNet110}
    \label{fig:ece_var_cifar10_resnet110}
\end{subfigure}
\begin{subfigure}{.49\textwidth}
    \includegraphics[width=0.9\textwidth]{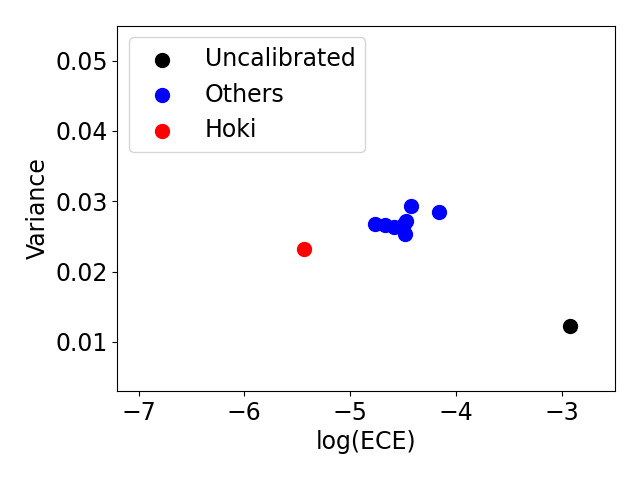}
    \caption{ResNet110 SD}
    \label{fig:ece_var_cifar10_resnet110sd}
\end{subfigure}
\\
\begin{subfigure}{.49\textwidth}
    \includegraphics[width=0.9\textwidth]{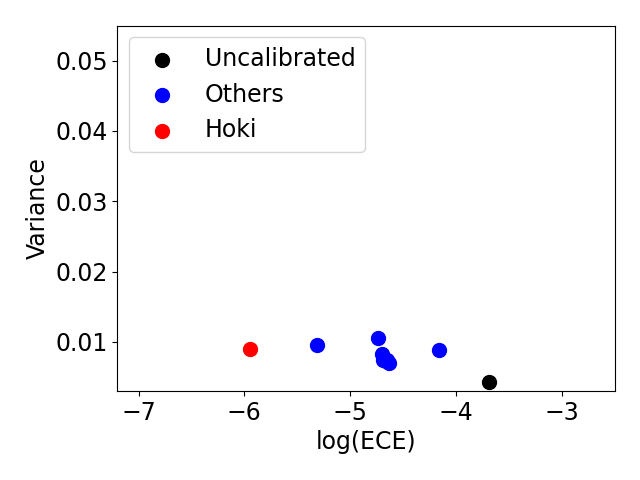}
    \caption{WRN 28-10}
    \label{fig:ece_var_cifar10_wrn28-10}
\end{subfigure}
\caption{log(ECE) vs. Variance on CIFAR10.}
\label{fig:ece_var_cifar10}
\end{figure}

\begin{figure}[!hbt]
    \centering
\begin{subfigure}{.49\textwidth}
    \includegraphics[width=0.9\textwidth]{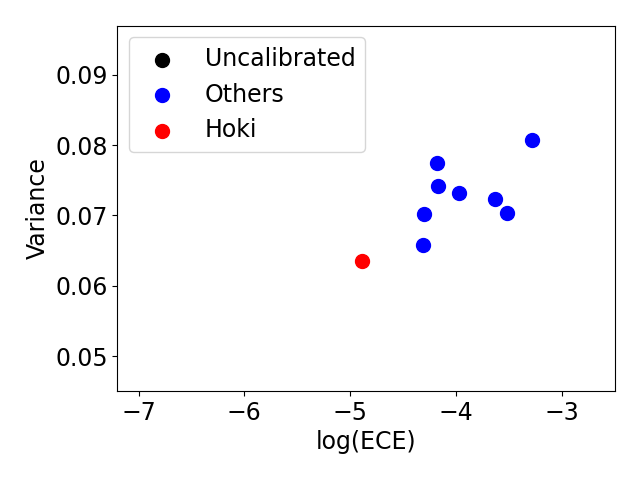}
    \caption{DenseNet40}
    \label{fig:ece_var_cifar100_densenet40}
\end{subfigure}
\begin{subfigure}{.49\textwidth}
    \includegraphics[width=0.9\textwidth]{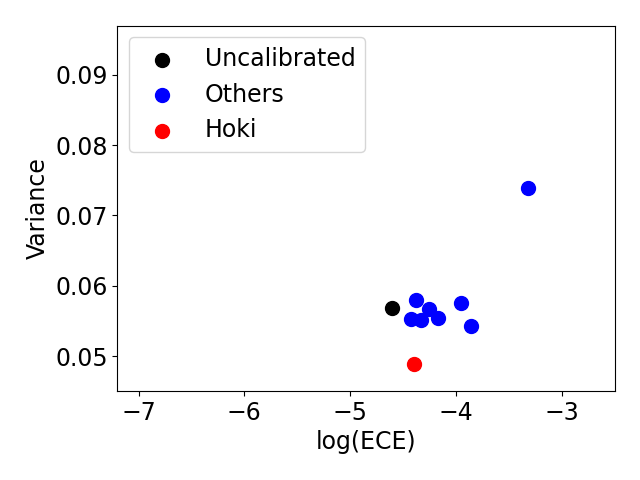}
    \caption{LeNet5}
    \label{fig:ece_var_cifar100_lenet5}
\end{subfigure}
\\
\begin{subfigure}{.49\textwidth}
    \includegraphics[width=0.9\textwidth]{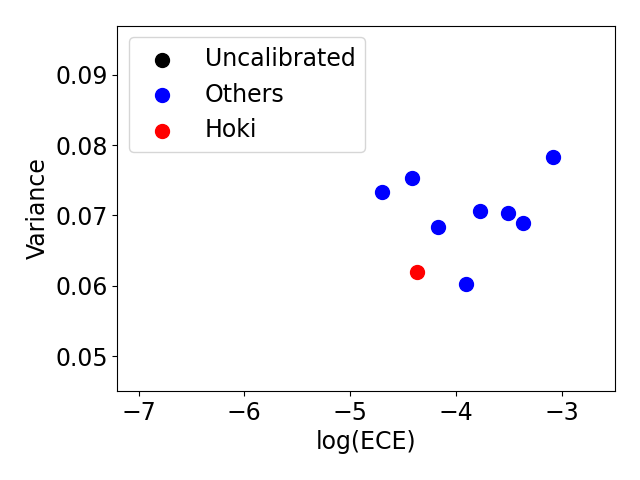}
    \caption{ResNet110}
    \label{fig:ece_var_cifar100_resnet110}
\end{subfigure}
\begin{subfigure}{.49\textwidth}
    \includegraphics[width=0.9\textwidth]{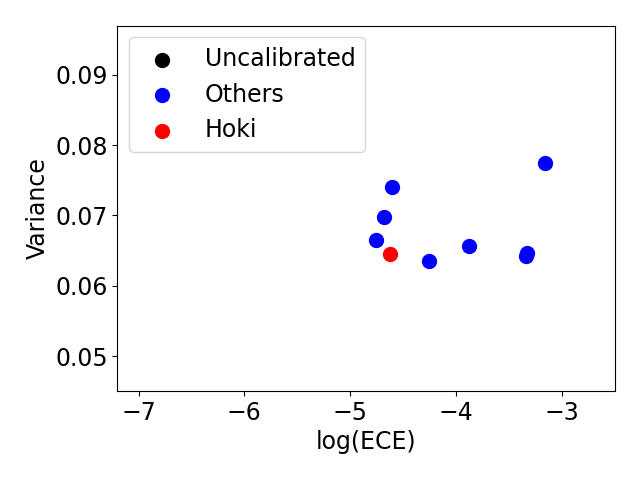}
    \caption{ResNet110 SD}
    \label{fig:ece_var_cifar100_resnet110sd}
\end{subfigure}
\\
\begin{subfigure}{.49\textwidth}
    \includegraphics[width=0.9\textwidth]{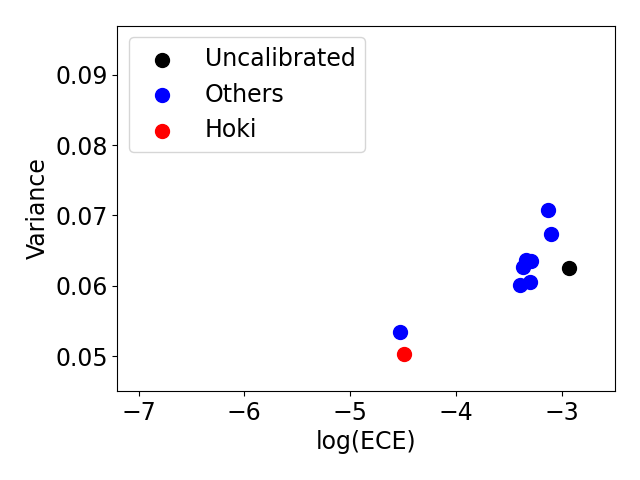}
    \caption{WRN 28-10}
    \label{fig:ece_var_cifar100_wrn28-10}
\end{subfigure}
\caption{log(ECE) vs. Variance on CIFAR100.}
\label{fig:ece_var_cifar100}
\end{figure}

\clearpage

\subsection{Number of bins}
In addition to the general setting for computing ECE (15 bins), we also use different number of bins (5, 10, 30, 50, 100) for better evaluation. The purpose of this evaluation is to show sensitive each algorithm is to the bin size -- a larger sensitivity would imply that an algorithm might not generalize as well on new data. 
As shown in Tables~\ref{tab:ece_5_bins}-\ref{tab:ece_100_bins}, Hoki outperforms other algorithms in many benchmarks except the extreme cases, using 5 bins and 100 bins.
We emphasize that Hoki always produces the best or the second-best performance on the challenging ImageNet except for one case -- ECE using 100 bins with the ResNet152 model. 
This result shows that, although Hoki is calibrated with 15 bins, the calibration is shown to be effective with various bin sizes, which highlights the benefit of calibration using transformations.

\begin{table*}[!hbt]
\caption{ECE using 5 bins}
\label{tab:ece_5_bins}
\centering
\resizebox{\textwidth}{!}{
\begin{sc}\begin{tabular}{llcccccccccc}
\toprule
Dataset & Model & Uncal. & TempS & VecS & \thead{MS-\\ODIR} & \thead{Dir-\\ODIR} & ETS & IRM & \thead{IROvA-\\TS} & ReCal & \ours{}\\
\midrule
MNIST & LeNet 5 & 0.0074 & 0.0014 & 0.0013 & 0.0011 & 0.0016 & 0.0015 & \underline{0.0011} & 0.0015 & 0.0030 & \textbf{0.0008}\\
\midrule
CIFAR10 & DenseNet 40 & 0.0519 & 0.0043 & \textbf{0.0023} & 0.0040 & \underline{0.0035} & 0.0040 & 0.0037 & 0.0038 & 0.0064 & 0.0039\\
CIFAR10 & LeNet 5 & 0.0169 & 0.0074 & \underline{0.0065} & 0.0067 & 0.0068 & 0.0088 & \textbf{0.0065} & 0.0124 & 0.0192 & 0.0110\\
CIFAR10 & ResNet 110 & 0.0450 & 0.0066 & 0.0074 & \underline{0.0046} & 0.0083 & 0.0081 & 0.0051 & 0.0057 & 0.0103 & \textbf{0.0045}\\
CIFAR10 & ResNet 110 SD & 0.0534 & 0.0067 & 0.0053 & 0.0050 & \underline{0.0046} & 0.0091 & 0.0058 & 0.0063 & 0.0137 & \textbf{0.0044}\\
CIFAR10 & WRN 28-10 & 0.0248 & 0.0089 & 0.0095 & 0.0092 & 0.0092 & 0.0087 & 0.0115 & \underline{0.0027} & 0.0085 & \textbf{0.0026}\\
\midrule
CIFAR100 & DenseNet 40 & 0.1728 & 0.0126 & 0.0253 & 0.0285 & 0.0182 & 0.0106 & \underline{0.0085} & 0.0107 & 0.0344 & \textbf{0.0073}\\
CIFAR100 & LeNet 5 & \underline{0.0070} & 0.0206 & 0.0129 & 0.0105 & 0.0097 & 0.0180 & 0.0080 & \textbf{0.0060} & 0.0338 & 0.0123\\
CIFAR100 & ResNet 110 & 0.1422 & \textbf{0.0072} & 0.0300 & 0.0340 & 0.0190 & \underline{0.0109} & 0.0110 & 0.0147 & 0.0371 & 0.0127\\
CIFAR100 & ResNet 110 SD & 0.1229 & 0.0071 & 0.0358 & 0.0354 & 0.0207 & \underline{0.0067} & \textbf{0.0051} & 0.0125 & 0.0418 & 0.0084\\
CIFAR100 & WRN 28-10 & 0.0521 & 0.0422 & 0.0436 & 0.0354 & 0.0327 & 0.0364 & 0.0140 & \textbf{0.0088} & 0.0319 & \underline{0.0111}\\
\midrule
ImageNet & DenseNet 161 & 0.0564 & 0.0191 & 0.0211 & 0.0367 & 0.0477 & 0.0126 & 0.0093 & \underline{0.0068} & 0.0482 & \textbf{0.0043}\\
ImageNet & MobileNet V2 & 0.0266 & 0.0150 & 0.0135 & 0.0191 & 0.0204 & 0.0134 & 0.0060 & \underline{0.0038} & 0.0471 & \textbf{0.0010}\\
ImageNet & ResNet 152 & 0.0490 & 0.0201 & 0.0207 & 0.0347 & 0.0397 & 0.0139 & 0.0112 & \textbf{0.0031} & 0.0455 & \underline{0.0052}\\
ImageNet & WRN 101-2 & 0.0524 & 0.0307 & 0.0300 & 0.0413 & 0.0194 & 0.0245 & 0.0155 & \textbf{0.0065} & 0.0425 & \underline{0.0067}\\
\bottomrule
\end{tabular}
\end{sc}
}
\end{table*}

\begin{table*}[!hbt]
\caption{ECE using 10 bins}
\label{tab:ece_10_bins}
\centering
\resizebox{\textwidth}{!}{
\begin{sc}\begin{tabular}{llcccccccccc}
\toprule
Dataset & Model & Uncal. & TempS & VecS & \thead{MS-\\ODIR} & \thead{Dir-\\ODIR} & ETS & IRM & \thead{IROvA-\\TS} & ReCal & \ours{}\\
\midrule
MNIST & LeNet 5 & 0.0076 & 0.0014 & 0.0013 & \underline{0.0011} & 0.0016 & 0.0016 & 0.0014 & 0.0019 & 0.0030 & \textbf{0.0008}\\
\midrule
CIFAR10 & DenseNet 40 & 0.0519 & 0.0068 & \textbf{0.0029} & 0.0040 & \underline{0.0035} & 0.0056 & 0.0073 & 0.0073 & 0.0094 & 0.0057\\
CIFAR10 & LeNet 5 & 0.0177 & 0.0079 & 0.0076 & \textbf{0.0067} & \underline{0.0068} & 0.0110 & 0.0071 & 0.0142 & 0.0192 & 0.0110\\
CIFAR10 & ResNet 110 & 0.0457 & 0.0079 & 0.0082 & 0.0075 & 0.0083 & 0.0084 & \textbf{0.0054} & 0.0091 & 0.0106 & \underline{0.0071}\\
CIFAR10 & ResNet 110 SD & 0.0535 & 0.0088 & 0.0077 & 0.0050 & \underline{0.0046} & 0.0109 & 0.0072 & 0.0078 & 0.0145 & \textbf{0.0044}\\
CIFAR10 & WRN 28-10 & 0.0251 & 0.0090 & 0.0095 & 0.0092 & 0.0092 & 0.0087 & 0.0147 & \underline{0.0046} & 0.0085 & \textbf{0.0026}\\
\midrule
CIFAR100 & DenseNet 40 & 0.1728 & 0.0152 & 0.0258 & 0.0285 & 0.0182 & 0.0134 & \underline{0.0098} & 0.0124 & 0.0380 & \textbf{0.0073}\\
CIFAR100 & LeNet 5 & 0.0118 & 0.0208 & 0.0129 & \underline{0.0105} & \textbf{0.0097} & 0.0204 & 0.0131 & 0.0119 & 0.0357 & 0.0282\\
CIFAR100 & ResNet 110 & 0.1422 & \textbf{0.0074} & 0.0301 & 0.0340 & 0.0190 & \underline{0.0121} & 0.0138 & 0.0181 & 0.0419 & 0.0127\\
CIFAR100 & ResNet 110 SD & 0.1229 & \underline{0.0077} & 0.0358 & 0.0354 & 0.0207 & 0.0079 & \textbf{0.0066} & 0.0140 & 0.0418 & 0.0098\\
CIFAR100 & WRN 28-10 & 0.0524 & 0.0424 & 0.0446 & 0.0354 & 0.0327 & 0.0388 & 0.0303 & \textbf{0.0106} & 0.0319 & \underline{0.0112}\\
\midrule
ImageNet & DenseNet 161 & 0.0564 & 0.0203 & 0.0236 & 0.0367 & 0.0477 & 0.0126 & 0.0097 & \underline{0.0077} & 0.0486 & \textbf{0.0043}\\
ImageNet & MobileNet V2 & 0.0266 & 0.0156 & 0.0149 & 0.0191 & 0.0204 & 0.0138 & 0.0068 & \underline{0.0064} & 0.0474 & \textbf{0.0010}\\
ImageNet & ResNet 152 & 0.0490 & 0.0201 & 0.0207 & 0.0347 & 0.0397 & 0.0139 & 0.0112 & \underline{0.0055} & 0.0457 & \textbf{0.0052}\\
ImageNet & WRN 101-2 & 0.0524 & 0.0307 & 0.0310 & 0.0413 & 0.0194 & 0.0245 & 0.0155 & \underline{0.0076} & 0.0425 & \textbf{0.0067}\\
\bottomrule
\end{tabular}
\end{sc}
}
\end{table*}

\begin{table*}[!hbt]
\caption{ECE using 30 bins}
\label{tab:ece_30_bins}
\centering
\resizebox{\textwidth}{!}{
\begin{sc}\begin{tabular}{llcccccccccc}
\toprule
Dataset & Model & Uncal. & TempS & VecS & \thead{MS-\\ODIR} & \thead{Dir-\\ODIR} & ETS & IRM & \thead{IROvA-\\TS} & ReCal & \ours{}\\
\midrule
MNIST & LeNet 5 & 0.0078 & \underline{0.0025} & 0.0034 & 0.0037 & 0.0033 & 0.0035 & 0.0030 & 0.0032 & 0.0036 & \textbf{0.0008}\\
\midrule
CIFAR10 & DenseNet 40 & 0.0525 & 0.0100 & \textbf{0.0052} & 0.0088 & 0.0083 & 0.0108 & 0.0103 & 0.0110 & 0.0136 & \underline{0.0057}\\
CIFAR10 & LeNet 5 & 0.0230 & 0.0144 & 0.0153 & 0.0165 & 0.0145 & 0.0158 & \underline{0.0142} & 0.0181 & 0.0283 & \textbf{0.0110}\\
CIFAR10 & ResNet 110 & 0.0458 & 0.0098 & 0.0097 & 0.0098 & \underline{0.0091} & 0.0095 & \textbf{0.0086} & 0.0103 & 0.0156 & 0.0110\\
CIFAR10 & ResNet 110 SD & 0.0538 & 0.0135 & \underline{0.0110} & 0.0133 & \textbf{0.0103} & 0.0136 & 0.0130 & 0.0145 & 0.0172 & 0.0229\\
CIFAR10 & WRN 28-10 & 0.0255 & 0.0110 & 0.0103 & 0.0100 & 0.0105 & 0.0098 & 0.0164 & \underline{0.0051} & 0.0100 & \textbf{0.0026}\\
\midrule
CIFAR100 & DenseNet 40 & 0.1728 & 0.0201 & 0.0317 & 0.0324 & 0.0224 & 0.0179 & 0.0219 & \underline{0.0150} & 0.0386 & \textbf{0.0075}\\
CIFAR100 & LeNet 5 & \underline{0.0174} & 0.0224 & 0.0183 & 0.0192 & 0.0232 & 0.0231 & 0.0197 & \textbf{0.0144} & 0.0373 & 0.0282\\
CIFAR100 & ResNet 110 & 0.1423 & \textbf{0.0122} & 0.0319 & 0.0351 & 0.0231 & 0.0132 & 0.0190 & 0.0223 & 0.0470 & \underline{0.0127}\\
CIFAR100 & ResNet 110 SD & 0.1230 & 0.0119 & 0.0367 & 0.0362 & 0.0208 & 0.0173 & \underline{0.0108} & 0.0187 & 0.0449 & \textbf{0.0098}\\
CIFAR100 & WRN 28-10 & 0.0538 & 0.0449 & 0.0453 & 0.0378 & 0.0358 & 0.0392 & 0.0397 & \underline{0.0162} & 0.0366 & \textbf{0.0112}\\
\midrule
ImageNet & DenseNet 161 & 0.0564 & 0.0203 & 0.0244 & 0.0373 & 0.0477 & 0.0139 & 0.0111 & \underline{0.0105} & 0.0500 & \textbf{0.0085}\\
ImageNet & MobileNet V2 & 0.0280 & 0.0188 & 0.0193 & 0.0230 & 0.0275 & 0.0183 & 0.0110 & \underline{0.0089} & 0.0479 & \textbf{0.0011}\\
ImageNet & ResNet 152 & 0.0494 & 0.0211 & 0.0223 & 0.0359 & 0.0399 & 0.0144 & 0.0150 & \underline{0.0099} & 0.0460 & \textbf{0.0052}\\
ImageNet & WRN 101-2 & 0.0532 & 0.0321 & 0.0330 & 0.0420 & 0.0305 & 0.0262 & 0.0234 & \underline{0.0105} & 0.0428 & \textbf{0.0067}\\
\bottomrule
\end{tabular}
\end{sc}
}
\end{table*}

\begin{table*}[!hbt]
\caption{ECE using 50 bins}
\label{tab:ece_50_bins}
\centering
\resizebox{\textwidth}{!}{
\begin{sc}\begin{tabular}{llcccccccccc}
\toprule
Dataset & Model & Uncal. & TempS & VecS & \thead{MS-\\ODIR} & \thead{Dir-\\ODIR} & ETS & IRM & \thead{IROvA-\\TS} & ReCal & \ours{}\\
\midrule
MNIST & LeNet 5 & 0.0080 & 0.0044 & 0.0039 & 0.0037 & 0.0038 & 0.0032 & 0.0047 & \textbf{0.0026} & 0.0038 & \underline{0.0029}\\
\midrule
CIFAR10 & DenseNet 40 & 0.0524 & 0.0111 & 0.0116 & 0.0108 & \underline{0.0103} & 0.0135 & 0.0111 & 0.0113 & 0.0152 & \textbf{0.0064}\\
CIFAR10 & LeNet 5 & 0.0247 & 0.0182 & 0.0191 & 0.0194 & 0.0208 & 0.0184 & \underline{0.0179} & 0.0207 & 0.0312 & \textbf{0.0128}\\
CIFAR10 & ResNet 110 & 0.0462 & 0.0123 & 0.0113 & 0.0113 & 0.0107 & 0.0108 & \underline{0.0104} & 0.0105 & 0.0172 & \textbf{0.0071}\\
CIFAR10 & ResNet 110 SD & 0.0541 & 0.0146 & 0.0133 & 0.0142 & \textbf{0.0115} & 0.0161 & 0.0152 & \underline{0.0131} & 0.0181 & 0.0168\\
CIFAR10 & WRN 28-10 & 0.0255 & 0.0131 & 0.0120 & 0.0112 & 0.0109 & 0.0105 & 0.0169 & \underline{0.0054} & 0.0122 & \textbf{0.0026}\\
\midrule
CIFAR100 & DenseNet 40 & 0.1728 & 0.0244 & 0.0330 & 0.0339 & 0.0295 & 0.0240 & 0.0270 & \underline{0.0163} & 0.0435 & \textbf{0.0114}\\
CIFAR100 & LeNet 5 & 0.0241 & 0.0261 & \underline{0.0205} & 0.0230 & 0.0270 & 0.0281 & 0.0278 & \textbf{0.0148} & 0.0407 & 0.0455\\
CIFAR100 & ResNet 110 & 0.1423 & \textbf{0.0156} & 0.0322 & 0.0379 & 0.0256 & 0.0197 & 0.0234 & 0.0232 & 0.0515 & \underline{0.0192}\\
CIFAR100 & ResNet 110 SD & 0.1232 & \textbf{0.0164} & 0.0385 & 0.0398 & 0.0227 & 0.0233 & 0.0209 & 0.0221 & 0.0490 & \underline{0.0167}\\
CIFAR100 & WRN 28-10 & 0.0562 & 0.0465 & 0.0470 & 0.0387 & 0.0376 & 0.0413 & 0.0404 & \underline{0.0171} & 0.0387 & \textbf{0.0152}\\
\midrule
ImageNet & DenseNet 161 & 0.0567 & 0.0227 & 0.0257 & 0.0380 & 0.0478 & 0.0176 & 0.0147 & \underline{0.0131} & 0.0504 & \textbf{0.0044}\\
ImageNet & MobileNet V2 & 0.0299 & 0.0187 & 0.0204 & 0.0241 & 0.0287 & 0.0194 & 0.0145 & \underline{0.0094} & 0.0497 & \textbf{0.0016}\\
ImageNet & ResNet 152 & 0.0499 & 0.0239 & 0.0249 & 0.0365 & 0.0399 & 0.0177 & 0.0179 & \textbf{0.0107} & 0.0474 & \underline{0.0177}\\
ImageNet & WRN 101-2 & 0.0544 & 0.0341 & 0.0342 & 0.0428 & 0.0342 & 0.0271 & 0.0212 & \underline{0.0106} & 0.0438 & \textbf{0.0067}\\
\bottomrule
\end{tabular}
\end{sc}
}
\end{table*}

\begin{table*}[!hbt]
\caption{ECE using 100 bins}
\label{tab:ece_100_bins}
\centering
\resizebox{\textwidth}{!}{
\begin{sc}\begin{tabular}{llcccccccccc}
\toprule
Dataset & Model & Uncal. & TempS & VecS & \thead{MS-\\ODIR} & \thead{Dir-\\ODIR} & ETS & IRM & \thead{IROvA-\\TS} & ReCal & \ours{}\\
\midrule
MNIST & LeNet 5 & 0.0086 & 0.0057 & 0.0053 & 0.0052 & 0.0057 & 0.0054 & 0.0062 & \textbf{0.0032} & \underline{0.0045} & 0.0053\\
\midrule
CIFAR10 & DenseNet 40 & 0.0538 & 0.0156 & 0.0166 & 0.0146 & 0.0158 & 0.0176 & 0.0162 & \underline{0.0125} & 0.0171 & \textbf{0.0069}\\
CIFAR10 & LeNet 5 & 0.0285 & 0.0268 & \underline{0.0243} & 0.0259 & 0.0273 & 0.0289 & 0.0279 & \textbf{0.0213} & 0.0353 & 0.0336\\
CIFAR10 & ResNet 110 & 0.0474 & 0.0150 & 0.0148 & 0.0148 & 0.0142 & 0.0141 & 0.0146 & \underline{0.0110} & 0.0208 & \textbf{0.0076}\\
CIFAR10 & ResNet 110 SD & 0.0551 & 0.0196 & 0.0179 & 0.0189 & \underline{0.0165} & 0.0195 & 0.0214 & \textbf{0.0150} & 0.0213 & 0.0282\\
CIFAR10 & WRN 28-10 & 0.0261 & 0.0146 & 0.0143 & 0.0148 & 0.0152 & 0.0121 & 0.0201 & \textbf{0.0061} & 0.0132 & \underline{0.0110}\\
\midrule
CIFAR100 & DenseNet 40 & 0.1731 & 0.0292 & 0.0369 & 0.0393 & 0.0333 & 0.0340 & 0.0329 & \textbf{0.0174} & 0.0518 & \underline{0.0189}\\
CIFAR100 & LeNet 5 & 0.0321 & 0.0346 & 0.0303 & \underline{0.0296} & 0.0297 & 0.0338 & 0.0340 & \textbf{0.0166} & 0.0497 & 0.0455\\
CIFAR100 & ResNet 110 & 0.1425 & \underline{0.0264} & 0.0356 & 0.0416 & 0.0307 & 0.0264 & 0.0304 & \textbf{0.0257} & 0.0562 & 0.0456\\
CIFAR100 & ResNet 110 SD & 0.1235 & 0.0248 & 0.0408 & 0.0424 & 0.0293 & 0.0318 & 0.0281 & \underline{0.0232} & 0.0537 & \textbf{0.0167}\\
CIFAR100 & WRN 28-10 & 0.0596 & 0.0486 & 0.0509 & 0.0455 & 0.0405 & 0.0449 & 0.0439 & \textbf{0.0173} & 0.0452 & \underline{0.0206}\\
\midrule
ImageNet & DenseNet 161 & 0.0575 & 0.0256 & 0.0292 & 0.0392 & 0.0483 & 0.0205 & 0.0195 & \underline{0.0136} & 0.0536 & \textbf{0.0050}\\
ImageNet & MobileNet V2 & 0.0316 & 0.0228 & 0.0242 & 0.0274 & 0.0323 & 0.0248 & 0.0194 & \underline{0.0113} & 0.0538 & \textbf{0.0096}\\
ImageNet & ResNet 152 & 0.0509 & 0.0264 & 0.0281 & 0.0384 & 0.0405 & 0.0214 & \underline{0.0201} & \textbf{0.0129} & 0.0508 & 0.0236\\
ImageNet & WRN 101-2 & 0.0554 & 0.0356 & 0.0373 & 0.0441 & 0.0362 & 0.0308 & 0.0267 & \underline{0.0127} & 0.0467 & \textbf{0.0085}\\
\bottomrule
\end{tabular}
\end{sc}
}
\end{table*}

\clearpage

\subsection{Initialization with Original Uncalibrated Confidence}
We perform an experiment to investigate the effect of different initializations.
In Algorithm \ref{alg:design_time} and \ref{alg:run_time}, we initialize the confidence with the validation set accuracy.
We can also use the original uncalibrated confidence from a classifier as the initial value, and we compare ECE values with those two different initialization.
Table \ref{tab:exp_init} shows that the initialization with the validation set accuracy is always better than the initialization with original uncalibrated confidence except two benchmarks, (CIFAR10, DenseNet40) and (CIFAR100, ResNet110 SD).
This difference illustrates the importance of the initialization of Hoki -- starting from a high-variance initial set of confidences may make it harder to converge to a good local optimum in terms of ECE.

\begin{table*}[!hbt]
\caption{{ECE by different initialization}}
\label{tab:exp_init}
\vskip 0.15in
\begin{sc}\begin{center}
\begin{tabular}{llcccccc}
\toprule
Dataset & Model & \thead{Val\\Accuracy} & \thead{Uncalibrated\\Confidence}\\
\midrule
\multirow{1}{*}{MNIST} & LeNet 5 & \textbf{0.0008} & 0.0018\\
\midrule
\multirow{5}{*}{CIFAR10} & DenseNet 40 & 0.0057 & \textbf{0.0038}\\
 & LeNet 5 & \textbf{0.0110} & 0.0171\\
 & ResNet 110 & \textbf{0.0071} & 0.0093\\
 & ResNet 110 SD & \textbf{0.0044} & 0.0060\\
 & WRN 28-10 & \textbf{0.0026} & 0.0042\\
\midrule
\multirow{5}{*}{CIFAR100} & DenseNet 40 & \textbf{0.0073} & 0.0178\\
 & LeNet 5 & \textbf{0.0123} & 0.0189\\
 & ResNet 110 & \textbf{0.0127} & 0.0157\\
 & ResNet 110 SD & 0.0098 & \textbf{0.0090}\\
 & WRN 28-10 & \textbf{0.0112} & 0.0117\\
\midrule
\multirow{4}{*}{ImageNet} & DenseNet 161 & \textbf{0.0043} & 0.0069\\
 & MobileNet V2 & \textbf{0.0011} & 0.0061\\
 & ResNet 152 & \textbf{0.0052} & 0.0081 \\
 & WRN 101-2 & \textbf{0.0067} & 0.0077\\
\bottomrule
\end{tabular}
\end{center}
\end{sc}
\end{table*}
\clearpage

\section{Computing Environment}
All experiments were run on a server with the specifications described in Table~\ref{tab:computing_spec}.
\begin{table}[!hbt]
    \centering
    \caption{Computing Specification}
    \label{tab:computing_spec}

    \begin{tabular}{cc}
    \toprule
    Item & Specification \\
    \midrule
    CPU &  Intel(R) Xeon(R) Gold 6248 CPU @ 2.50GHz \\
    Memory & 768 GB\\
    GPU & NVIDIA GeForce RTX 2080 Ti\\
    \bottomrule
    \end{tabular}
\end{table}